\documentclass{article}

\PassOptionsToPackage{numbers, compress}{natbib}

\usepackage[final]{nips_2018}

\usepackage[utf8]{inputenc} 
\usepackage[T1]{fontenc}    
\usepackage{hyperref}       
\usepackage{url}            
\usepackage{booktabs}       
\usepackage{amsfonts}       
\usepackage{nicefrac}       
\usepackage{microtype}      

\usepackage{multicol}
\usepackage[utf8]{inputenc}
\usepackage{times}
\usepackage{graphicx} 
\usepackage{subfigure} 
\usepackage{amsmath}

\usepackage{algorithm}
\usepackage{algorithmic}

\usepackage{amssymb,amsmath,amsfonts}
\usepackage{mathrsfs}
\usepackage{color}
\newcommand{\norm}[1]{\left\lVert#1\right\rVert}
\newcommand{\expect}[1]{\mathbb{E}\left[{#1}\right]}
\newcommand{\prob}[1]{\mathbb{P}\left[{#1}\right]}
\newcommand{\given}{\; \big\vert \;} 
\newcommand{\bydef}{:=}

\newcommand{\inner}[2]{\langle #1, #2 \rangle}

\usepackage{subfigure}
\usepackage{mathnotations2}

\graphicspath{{./Figures/}}

\newtheorem{mytheorem}{Theorem}

\newtheorem{mylemma}{Lemma}
\newtheorem{mycorollary}{Corollary}

\newtheorem{mydefinition}{Definition}

\newcommand{\BlackBox}{\rule{1.5ex}{1.5ex}}
\newenvironment{proof}{\par\noindent{\bf Proof\ }}{\hfill\BlackBox\\[2mm]}
\newcommand{\beq}{\begin{equation}}
\newcommand{\eeq}{\end{equation}}
\newcommand{\beqn}{\begin{equation*}}
\newcommand{\eeqn}{\end{equation*}}
\newcommand{\beqa}{\begin{eqnarray}}
\newcommand{\eeqa}{\end{eqnarray}}
\newcommand{\beqan}{\begin{eqnarray*}}
\newcommand{\eeqan}{\end{eqnarray*}}

\newcommand{\argmax}{\mathop{\mathrm{argmax}}}

\newcommand{\argmin}{\mathop{\mathrm{argmin}}}

\title{On Batch Bayesian Optimization}

\author{
  Sayak Ray Chowdhury\\
  Department of ECE\\
  Indian Institute of Science\\
  Bangalore, India 560012\\
  \texttt{sayak@iisc.ac.in} 
  \And
  Aditya Gopalan \\
  Department of ECE\\
  Indian Institute of Science\\
  Bangalore, India 560012 \\
  \texttt{aditya@iisc.ac.in} 
}

\begin{document}

\maketitle

\begin{abstract}
We present two algorithms for Bayesian optimization in the batch feedback setting, based on Gaussian process upper confidence bound and Thompson sampling approaches, along with frequentist regret guarantees and numerical results.
\end{abstract}
 
\vspace*{-3mm}
\section{Introduction}
\label{sec:Introduction}
\vspace*{-3mm}
Black-box optimization of an unknown function is an important problem in several real world domains such as hyper-parameter tuning of complex machine learning models, experimental design etc, and in recent years, the Bayesian optimization framework has gained a lot of traction towards achieving this goal.
Bayesian optimization (BO) methods start with a prior distribution, generally Gaussian processes (GPs), over a function class, and use function evaluations to compute the posterior distribution. Popular strategies in this vein include
expected improvement (GP-EI) \cite{movckus1975bayesian}, probability of improvement (GP-PI) \cite{wang2016optimization}, upper confidence bounds (GP-UCB) \cite{srinivas2012information}, Thompson sampling (GP-TS) \cite{pmlr-v70-chowdhury17a}, predictive-entropy search \cite{hernandez2015predictive}, etc. In some cases, it is possible and also desirable to evaluate the function in batches, e.g.,  parallelizing an expensive computer simulation over multiple cores. In this case, we can gather more information within a same time window, but future decisions need to be taken without the benefit of the evaluations in progress. A flurry of parallel (or, equivalently batch) Bayesian optimization strategies have been developed recently to address this problem \cite{desautels2014parallelizing,kandasamy2018parallelised,contal2013parallel,kathuria2016batched,gonzalez2016batch}. In this work, we explore further the potential of batch BO strategies for black-box optimization, assuming that the unknown function is in the Reproducing Kernel Hilbert Space (RKHS) induced by a symmetric positive semi-definite kernel.

\vspace*{-1mm}
\textbf{Contributions.} 
We design a new algorithm – Improved Gaussian
Process-Batch Upper Confidence Bound (IGP-BUCB) – for batch Bayesian optimization. It
is a variant of the GP-BUCB algorithm of \citet{desautels2014parallelizing},
but with a significantly reduced confidence interval resulting in an order-wise improvement in its regret bound. We also develop a nonparametric version of Thompson sampling, namely Gaussian Process-Batch Thompson Sampling (GP-BTS), and
prove the first frequentist guarantee of TS in the setting of Batch Bayesian optimization. To put this in perspective, GP-BTS can be seen as a variant of the AsyTS algorithm of \citet{kandasamy2018parallelised}. But the setting under which it is analyzed in this work is agnostic i.e., under a fixed but unknown function, whereas \citet{kandasamy2018parallelised} consider the pure Bayesian setup. Finally, we confirm empirically the efficiency of IGP-BUCB and GP-BTS on several synthetic and real-world datasets.
\vspace*{-3mm}
\section{Problem Statement}
\label{sec:Problem-Statement}
\vspace*{-3mm}
We consider the problem of sequentially maximizing a fixed but unknown reward function $f:D\ra \Real$ over a set of decisions (equivalently arms or actions) $D \subset \Real^d$. An algorithm for this problem chooses, at each round $t$, an action $x_t \in D$, and observes a noisy reward $y_t=f(x_t) + \epsilon_t$. We assume that the noise sequence $\lbrace\epsilon_t\rbrace_{t \ge 1}$ is conditionally $R$-sub-Gaussian for a fixed constant $R \ge 0$, i.e., for all $t \ge 0$ and $\lambda \in \Real$, $\expect{e^{\lambda \epsilon_t} \given \cF_{t-1}} \le \exp\left(\lambda^2R^2/2\right)$,
where $\cF_{t-1}$ is the $\sigma$-algebra generated by the random variables $\lbrace x_s, \epsilon_s\rbrace_{s = 1}^{t-1}$ and $x_t$. The decision $x_t$ is chosen causally depending upon the arms played and rewards available till round $t-1$. Specifically, for each decision round $t$, let $\cS(t) \le t-1$ represent the index of the most recent round for which rewards are available, so that $x_t$ can be chosen using only rewards obtained till round $\cS(t)$, along with actions (naturally) known to the algorithm until round $t-1$. We assume that $t-\cS(t) \le M$ for a known constant $M \ge 1$, i.e., rewards are available as batches of variable lengths upto $M$. For example: (a) if $\cS(t)=M\lfloor (t-1)/M \rfloor$, then rewards are available as batches of length $M$ and it is denoted as the \textit{simple batch} setting and (b) if $\cS(t)=\max\lbrace t-M,0 \rbrace$, then the rewards are delayed by $M$ time periods and it is denoted as the \textit{simple delay} setting. An important special case is when $M=1$ or equivalently, $\cS(t)=t-1$. Then all the rewards till round $t-1$ are available, and this represents the standard \textit{strictly sequential} setting.
%

\vspace*{-1mm}
\textbf{Regret.} A natural goal of a sequential algorithm is to maximize its cumulative reward $\sum_{t=1}^{T} f(x_t)$ over a time horizon $T$ or equivalently minimize its cumulative {\em regret} $R_T = \sum_{t=1}^T \left(f(x^\star)-f(x_t)\right)$, where $x^\star \in \argmax_{x\in D}f(x)$ is a maximum point of $f$ (assuming the maximum is attained; not necessarily unique). A sublinear growth of $R_T$ in $T$ signifies that the time-average regret $R_T/T \ra 0$ as $T\ra \infty$. 

\vspace*{-1mm}
\textbf{Regularity assumptions.} Attaining sub-linear regret is impossible in general for arbitrary reward functions $f$, and thus some regularity assumptions are in order. In what follows, we assume that $f$ has small norm in the reproducing Kernel Hilbert space (RKHS), denoted as $\cH_k(D)$, of real valued functions on $D$, with positive semi-definite kernel function $k: D \times D \to \mathbb{R}$. %
We assume a known bound on the RKHS norm of $f$, i.e., $\norm{f}_k \le B$. Moreover, we assume bounded variance by restricting $k(x,x) \le 1$, for all $x \in D$. Some common kernels, such as
the Squared Exponential (SE) kernel and the Matérn kernel, satisfy this property.

\vspace*{-3mm}
\section{Algorithms}
\label{sec:Algorithms}
\vspace*{-3mm}
{\bf Representing uncertainty of $f$ via Gaussian processes.} 
We model $f$ as a sample from a
Gaussian process prior $GP_D(0,k)$, and assume that the noise variables $\epsilon_t \sim \cN(0,\lambda)$ are i.i.d. Gaussian. %
By standard properties of GPs \cite{rasmussen2006gaussian}, conditioned on the history of observations $\cH_t \bydef \lbrace (x_s,y_s)\rbrace_{s=1}^{t}$, the posterior over $f$ is also a Gaussian process, $GP_D(\mu_t,k_t)$, with mean function $\mu_t(x) \bydef k_t(x)^T(K_t + \lambda I)^{-1}Y_t$ and kernel function $k_t(x,x') \bydef k(x,x') - k_t(x)^T(K_t + \lambda I)^{-1} k_t(x')$. Here $Y_t \bydef [y_1,\ldots,y_t]^T$ denotes the vector of rewards observed at the set $A_t \bydef \lbrace x_1,\ldots,x_t \rbrace$, $k_t(x) \bydef  [k(x_1,x),\ldots,k(x_t,x)]^T$ denotes the vector of kernel evaluations between $x$ and elements of the set $A_t$ and $K_t \bydef [k(u,v)]_{u,v \in A_t}$ denotes the kernel matrix computed at $A_t$.

\vspace*{-1mm}
\textbf{Representing the posterior GP with delayed feedback.} In the batch (equivalently delayed) feedback setup, the only available rewards at the start of round $t$ are $y_1,\ldots,y_{\cS(t)}$; however, all the previous decisions $ x_1,\ldots,x_{t-1}$ are available. This suggests `hallucinating' the  missing rewards $y_{S(t)+1},\ldots,y_{t-1}$, an idea first proposed by  \citet{desautels2014parallelizing}, using the most recently updated
posterior mean $\mu_{\cS(t)}$, i.e., setting $y_s=\mu_{\cS(t)}(x_s)$ for all $s = S(t)+1, S(t)+2,\ldots,t-1 $. By doing this, observe via,  say, the iterative GP update equations (\ref{eqn:mean-online}) and (\ref{eqn:cov-online}) \cite{chowdhury2017kernelized}, that the
mean of the posterior including the hallucinated observations remains precisely $\mu_{\cS(t)}$, 
but the posterior covariance decreases to $k_{t-1}$.
\vspace*{-3mm}
\footnotesize
\beqa
\mu_s(x) &=& \mu_{s-1}(x) + \dfrac{k_{s-1}(x_s,x)}{\lambda+\sigma^2_{s-1}(x_s)}\left(y_s - \mu_{s-1}(x_s)\right),\label{eqn:mean-online}\\
k_s(x,x') &=& k_{s-1}(x,x') - \dfrac{k_{s-1}(x_s,x)k_{s-1}(x_s,x')}{\lambda+ \sigma^2_{s-1}(x_s)}.\label{eqn:cov-online}
\eeqa
\normalsize
Therefore, a natural approach towards batch Bayesian optimization is to use a decision rule that sequentially chooses actions using all the information that is available so far, i.e., a rule that uses  the most recently updated posterior mean $\mu_{\cS(t)}$
and posterior kernel $k_{t-1}$ to choose action $x_t$ at round $t$.

\vspace*{-1mm}
\textbf{Improved GP-Batch UCB (IGP-BUCB) algorithm.} IGP-BUCB (Algorithm \ref{algo:ucb}), at each round $t$, 
chooses the action $x_t=\argmax_{x\in D}\mu_{\cS(t)}(x)+\beta_t\sigma_{t-1}(x)$, where $\sigma_{t-1}(x) \bydef k_{t-1}(x,x)$ and $\beta_t\bydef \sqrt{\xi_M}\Big(B+ \dfrac{R}{\sqrt{\lambda}}\sqrt{2\left(\gamma_{\cS(t)}+\ln(1/\delta)\right)}\Big)$. Here $0 < \delta \le 1$ is a free parameter. $\gamma_t \bydef \max_{A \subset D : \abs{A}=t} I(f_A;Y_A)$
denotes the \textit{Maximum Information Gain} about any $f \sim GP_D(0,k)$ from $t$ noisy observations $Y_A$, which are obtained by passing $f_A := [f(x)]_{x\in A}$ through a channel $\prob{Y_A | f_A}=\cN(0,\lambda I)$. The key quantity $\xi_M$ bounds the information we gain about $f$ from the hallucinated observations (there are at most $M-1$ of them at every round) conditioned on the actual observations, in the sense that
$I\big(f(x);Y_{\cS(t)+1:t-1}\given Y_{1:\cS(t)}\big) \le 1/2 \ln(\xi_M)$ for all $x \in D$, where $Y_{1:\cS(t)} \bydef [y_1,\ldots,y_{\cS(t)}]^T$ and $Y_{\cS(t)+1:t-1} \bydef [y_{\cS(t)+1},\ldots,y_{t-1}]^T$ denote the vectors of actual and hallucinated observations, respectively.
This rule inherently
trades off exploration (picking points with high uncertainty
$\sigma_{t-1}(x)$) with exploitation (picking points with high reward $\mu_{\cS(t)}(x)$), where  
$\beta_t$ serves twin purposes: (a) it balances exploration and exploitation, (b) it compensates (via $\xi_M$) for the bias created by the hallucinated data $Y_{\cS(t)+1:t-1}$, in the attempt to aggressively shrink the confidence interval and reduce exploration. \\
\textbf{Note:} While \citet{desautels2014parallelizing} propose the GP-BUCB algorithm, which also uses the same template as IGP-BUCB, we are able to reduce the width of the confidence interval and provably improve upon regret (Section \ref{sec:main-results}).

\begin{minipage}[t]{7cm}
\vspace*{-5mm} 
  \vspace{0pt}  
  \begin{algorithm}[H]
   \renewcommand\thealgorithm{1}
   \caption{IGP-BUCB}\label{algo:ucb}
   \begin{algorithmic}
   \STATE \textbf{Input:} Kernel $k$, feedback mapping $\cS$.
   \FOR{$t = 1, 2, 3 \ldots$}
   \STATE Choose $x_t = \argmax\limits_{x\in D} \mu_{\cS(t)}(x) + \beta_t\sigma_{t-1}(x)$.
   \STATE Update $\sigma_t$ using (\ref{eqn:cov-online}).
   \IF{$\cS(t) < \cS(t+1)$}
   \FOR{$s = \cS(t)+1,\ldots,\cS(t+1)$}
   \STATE Observe $y_s = f(x_s)+\epsilon_s$.
   \STATE Update $\mu_s$ using (\ref{eqn:mean-online}).
   \ENDFOR
   \ENDIF
   \ENDFOR
   \end{algorithmic}
   \addtocounter{algorithm}{-1}
   \end{algorithm}
\end{minipage}%
\begin{minipage}[t]{7cm}
\vspace*{-5mm} 
  \vspace{0pt}
  \begin{algorithm}[H]
   \renewcommand\thealgorithm{2}
    \caption{GP-BTS}\label{algo:ts} \begin{algorithmic}
    \STATE \textbf{Input:} Kernel $k$, feedback mapping $\cS$.
    \FOR{$t = 1, 2, 3 \ldots$}
    \STATE Sample $f_t \sim GP_{D_t}(\mu_{\cS(t)},v_t^2k_{t-1})$.
    \STATE Choose $x_t = \argmax_{x\in D_t} f_t(x)$.
    \STATE Update $k_t$ using (\ref{eqn:cov-online}).
     \IF{$\cS(t) < \cS(t+1)$}
     \FOR{$s = \cS(t)+1,\ldots,\cS(t+1)$}
     \STATE Observe $y_s = f(x_s)+\epsilon_s$.
     \STATE Update $\mu_s$ using (\ref{eqn:mean-online}).
     \ENDFOR
     \ENDIF
    \ENDFOR
    \end{algorithmic}
    \addtocounter{algorithm}{-2}
    \end{algorithm}
\end{minipage}

\textbf{GP-Batch Thompson Sampling (GP-BTS) algorithm.} Thompson sampling is a randomized strategy, and at every round chooses the action according to the posterior probability that it is optimal. At every round $t$, GP-BTS (Algorithm \ref{algo:ts}) (a) samples a random function $f_t$ from the posterior Gaussian process $GP_{D_t}\left(\mu_{\cS(t)},v_t^2k_{t-1}\right)$, where $D_t$ is a suitable discretization (See Appendix \ref{appendix:TS} for details) of $D$, $v_t \bydef \sqrt{\xi_M}\Big(B + \dfrac{R}{\sqrt{\lambda}}\sqrt{2\left(\gamma_{\cS(t)}+\ln(2/\delta)\right)}\Big)$,  and (b) chooses the action $x_t = \argmax_{x\in D_t} f_t(x)$. Here again, $v_t$ plays a role similar to that of $\beta_t$ as in IGP-BUCB, i.e., promoting exploration and compensating for the bias of hallucination.

\textit{\textbf{Remark.}}
  One particular choice for $\xi_M$ is $e^{2\gamma_{M-1}}$ \cite{desautels2014parallelizing}, and $\gamma_t$ (or an upper bound on it) can be computed given the kernel \cite{srinivas2009gaussian}; e.g., for the Squared Exponential (SE) kernel, $\gamma_t = O\big((\ln t)^{d}\big)$ and for the Mat$\acute{e}$rn
kernel with smoothness parameter $\nu$, $\gamma_t = O\big( t^{\frac{d(d+1)}{2\nu+d(d+1)}}\ln t\big)$.

\vspace*{-3mm}
\section{Regret Bounds for IGP-BUCB and GP-BTS}
\label{sec:main-results}
\vspace*{-3mm}
Though our algorithms rely on GP priors, the setting under which they are
analyzed is {\em agnostic}, i.e., under a {\em fixed} (non-random) but unknown reward function. This is arguably more challenging \cite{srinivas2009gaussian} than traditional Bayesian regret (expected regret under a random reward function from the known GP prior) analysis. 
\vspace*{-3mm} 
\begin{mytheorem}[Regret bound for IGP-BUCB]
\label{thm:regret-bound-UCB-approx}
Let $D \subset \Real^d$, $f$ be a member of the RKHS $\cH_k(D)$, with $\norm{f}_k \le B$ and the noise sequence $\lbrace \epsilon_t \rbrace_{t \ge 1}$ be conditionally $R$-sub-Gaussian. Then, for any $0 < \delta \le 1$, IGP-BUCB enjoys, with probability at least $1-\delta$, the regret bound
$R_T = O\Big(B\sqrt{\xi_M T\gamma_T}+\sqrt{\xi_M T\gamma_T\big(\gamma_T+\ln(1/\delta)\big)}\Big)$.
\end{mytheorem}
\vspace*{-3mm}
The regret bound for IGP-BUCB is $O\Big(\sqrt{\xi_M T}\big(B\sqrt{\gamma_T}+\gamma_T\big)\Big)$ with high probability, whereas
\citet{desautels2014parallelizing} show that GP-BUCB obtains regret $O\Big(\sqrt{\xi_M T}\big(B\sqrt{\gamma_T}+\gamma_T\ln^{3/2}(T)\big)\Big)$ with high probability. Hence, we obtain a $O(\ln^{3/2}T)$ multiplicative factor \textbf{improvement} in the final regret bound;  our numerical experiments reflect this improvement. 
\vspace*{-2mm}
\begin{mytheorem}[Regret bound for GP-BTS] 
\label{thm:regret-bound-TS-approx}
Let $D \subset \Real^d$ be compact and convex, $f$ be a member of the RKHS $\cH_k(D)$, with $\norm{f}_k \le B$ and the noise sequence $\lbrace \epsilon_t \rbrace_{t \ge 1}$ be conditionally $R$-sub-Gaussian. Then, for any $0 < \delta \le 1$, GP-BTS enjoys, with probability at least $1-\delta$, the regret bound $R_T=O\Big(\sqrt{\xi_M \big(\gamma_T+\ln(2/\delta)\big)d\ln(BdT)}\big(\sqrt{T\gamma_T}+B\sqrt{T\ln(2/\delta)}\big)\Big)$.
\end{mytheorem}
\vspace*{-4mm}
The regret bound for GP-BTS is $O\Big(\sqrt{\xi_M Td\ln (BdT)}\left(B\sqrt{\gamma_T}+\gamma_T\right)\Big)$ with high probability. Though it is inferior to IGP-BUCB in terms of the dependency on dimension $d$, to the best of our knowledge, this represents the first
(frequentist) regret guarantee of Thompson sampling for batch Bayesian optimization.\\
\textbf{\textit{Remark.}} In the strictly sequential setup ($M=1$ and $\cS_t=t-1$), IGP-BUCB and GP-BTS reduce to the IGP-UCB and GP-TS algorithms of \citet{pmlr-v70-chowdhury17a}, respectively.

\vspace*{-1mm}
$\gamma_T$ is poly-logarithmic in $T$ for popular kernels \cite{srinivas2009gaussian}. Hence, the regret bounds of our algorithms grow sublinearly with $T$. But, if we naively run our algorithms with $\xi_M = \exp(2\gamma_{M-1})$ as discussed in Section \ref{sec:Algorithms}, then the regret bounds grow at least linearly with the batch size $M$. This can be obviated by incorporating the same initialization scheme (see Appendix \ref{appendix:init} for details) of \citet{desautels2014parallelizing}. 

\vspace*{-5mm}
\section{Experiments}
\vspace*{-3mm}
\label{sec:Experiments}
We numerically compare the performance of GP-BUCB \cite[Theorem 2, case 3]{desautels2014parallelizing} with our algorithms IGP-BUCB and GP-BTS in the kernelized setting. $\beta_t,v_t$ are set, unless otherwise specified, according to the theoretical bounds for the corresponding kernels, with $\delta=0.1, B= \max_{x\in D}\abs{f(x)}, \lambda=R^2=0.025$ and $\xi_M =1$ (similar to \cite{desautels2014parallelizing}). Unless otherwise specified, the time-average regret ($R_T/T$) of all algorithms in the simple batch setting (with $M=5$) are plotted in Figure \ref{fig:synthetic_plot_rkhs}. The experiments are performed on the following data:\\
\begin{figure}[t!]
\vskip -5mm
\centering
\subfigure[]{\includegraphics[height=1.15in,width=1.35in]{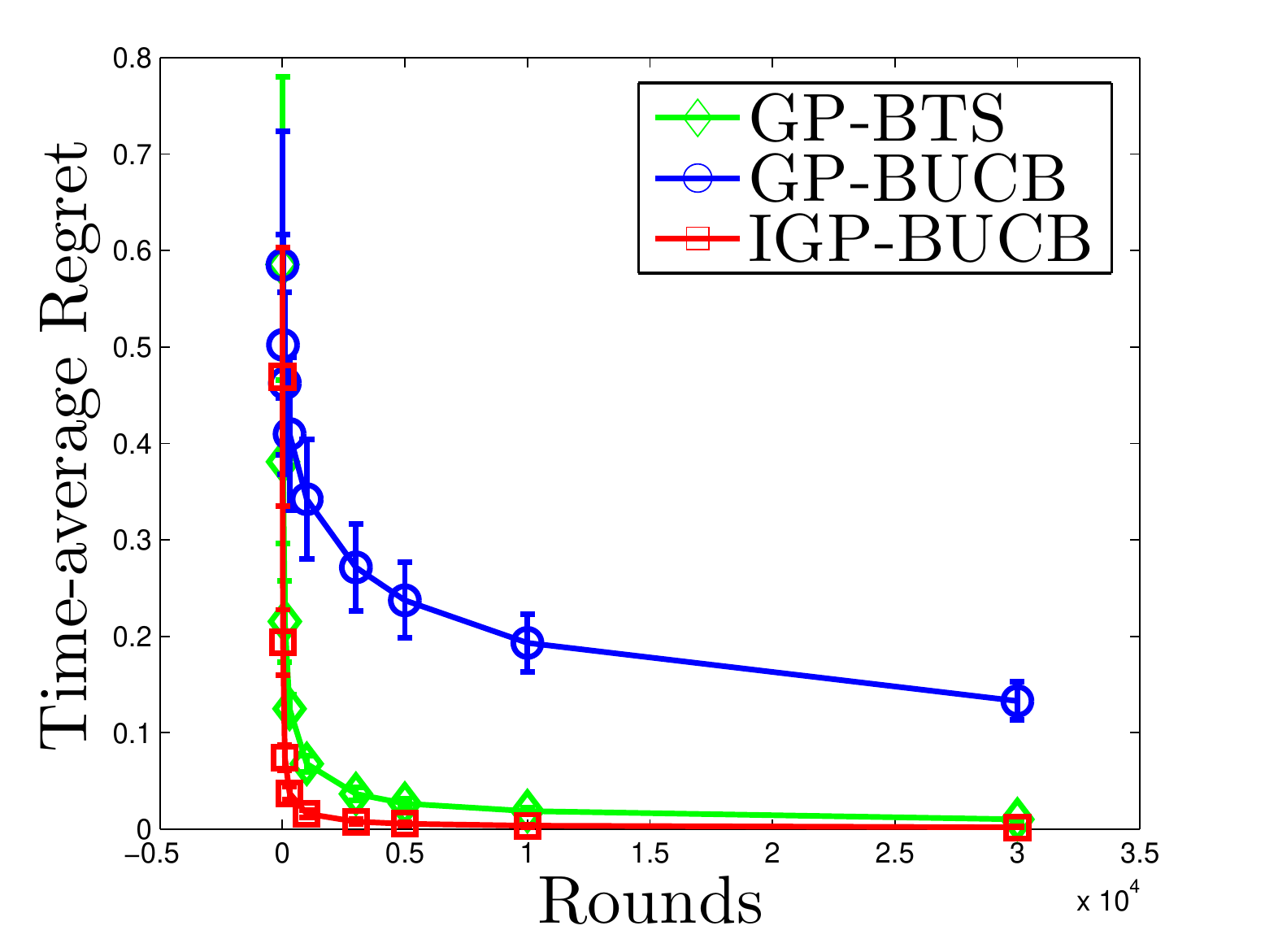}}
\subfigure[]{\includegraphics[height=1.15in,width=1.35in]{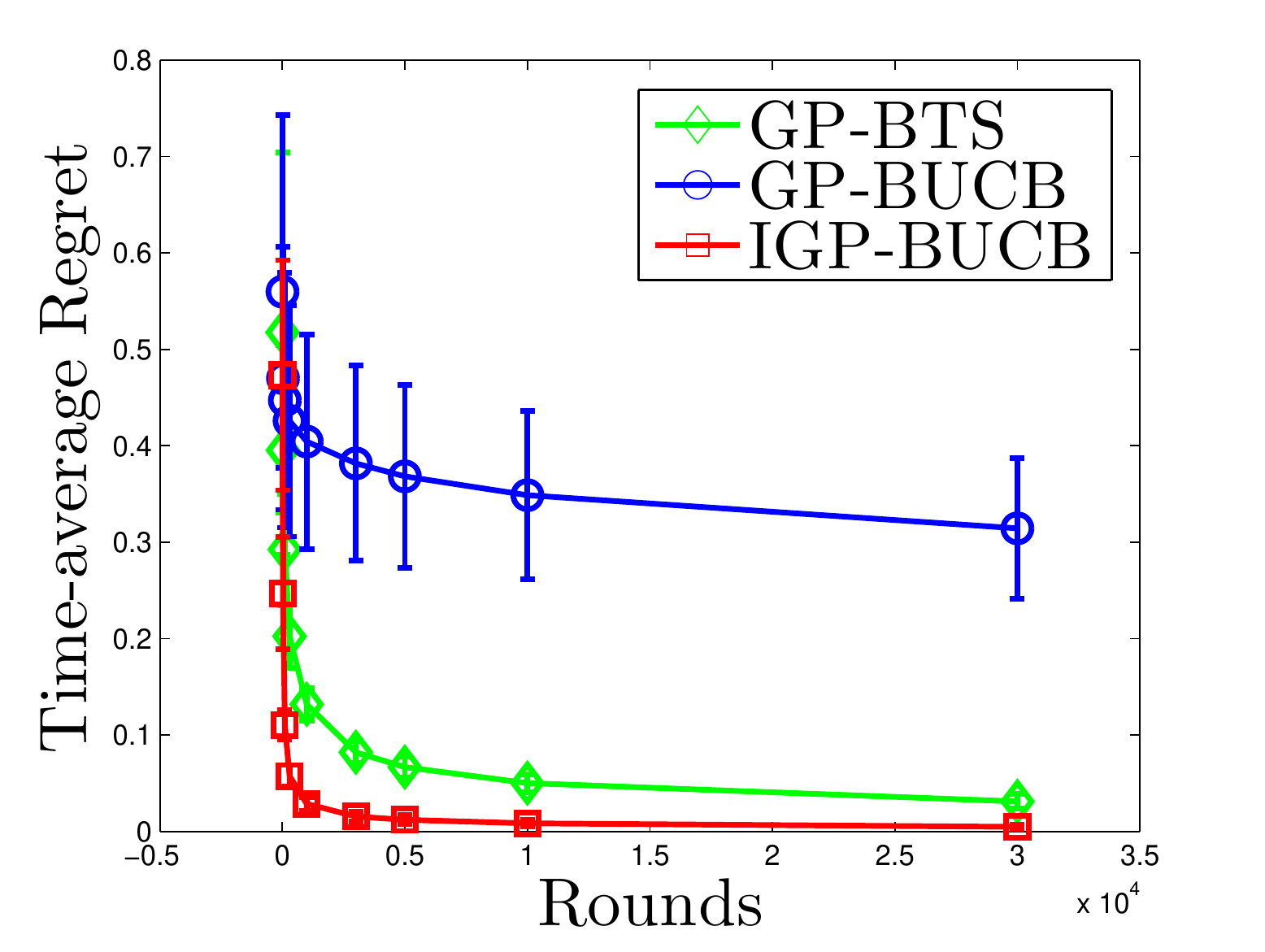}}
\subfigure[]{\includegraphics[height=1.15in,width=1.35in]{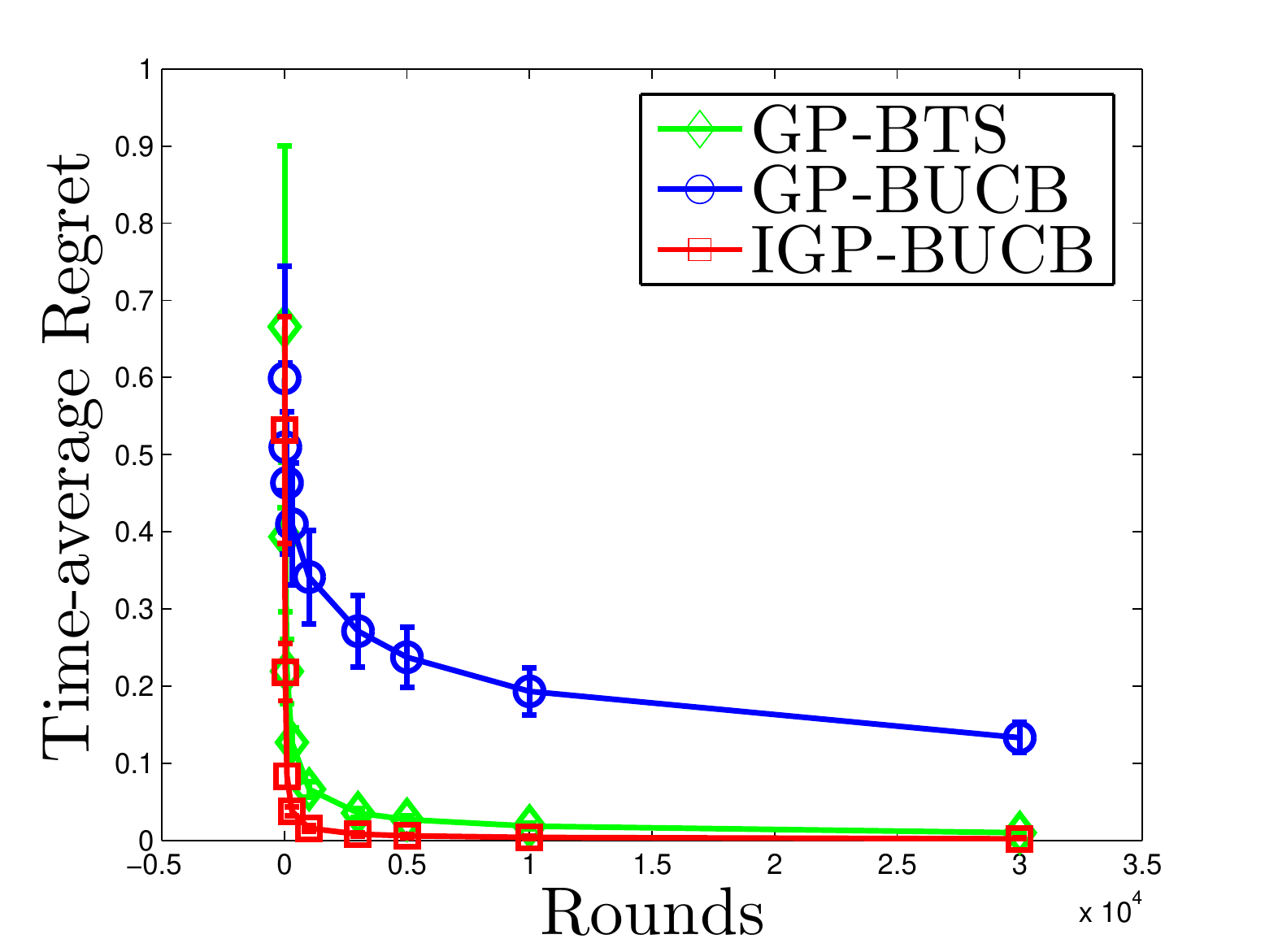}}
\subfigure[]{\includegraphics[height=1.15in,width=1.35in]{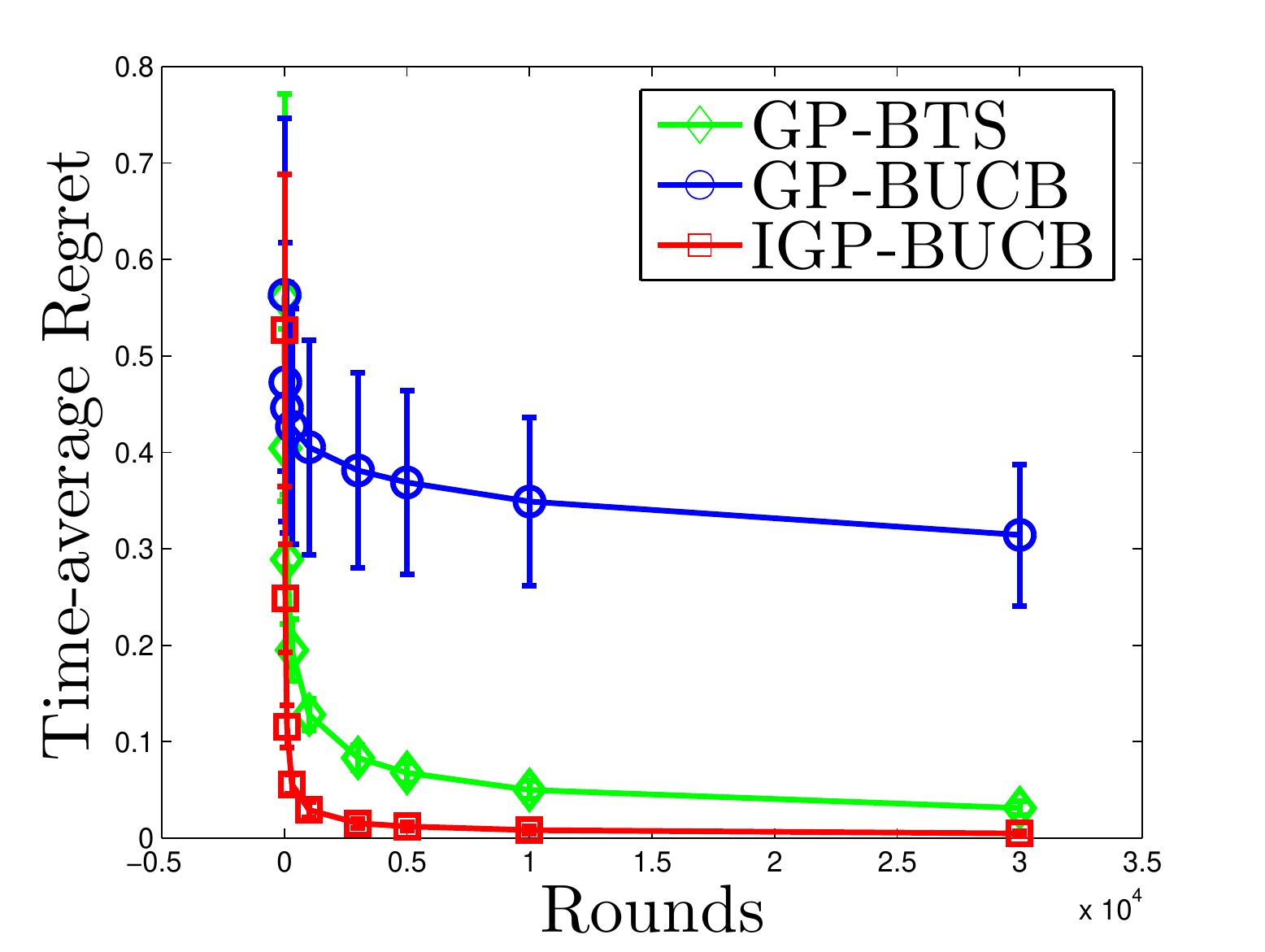}}
\subfigure[]{\includegraphics[height=1.15in,width=1.35in]{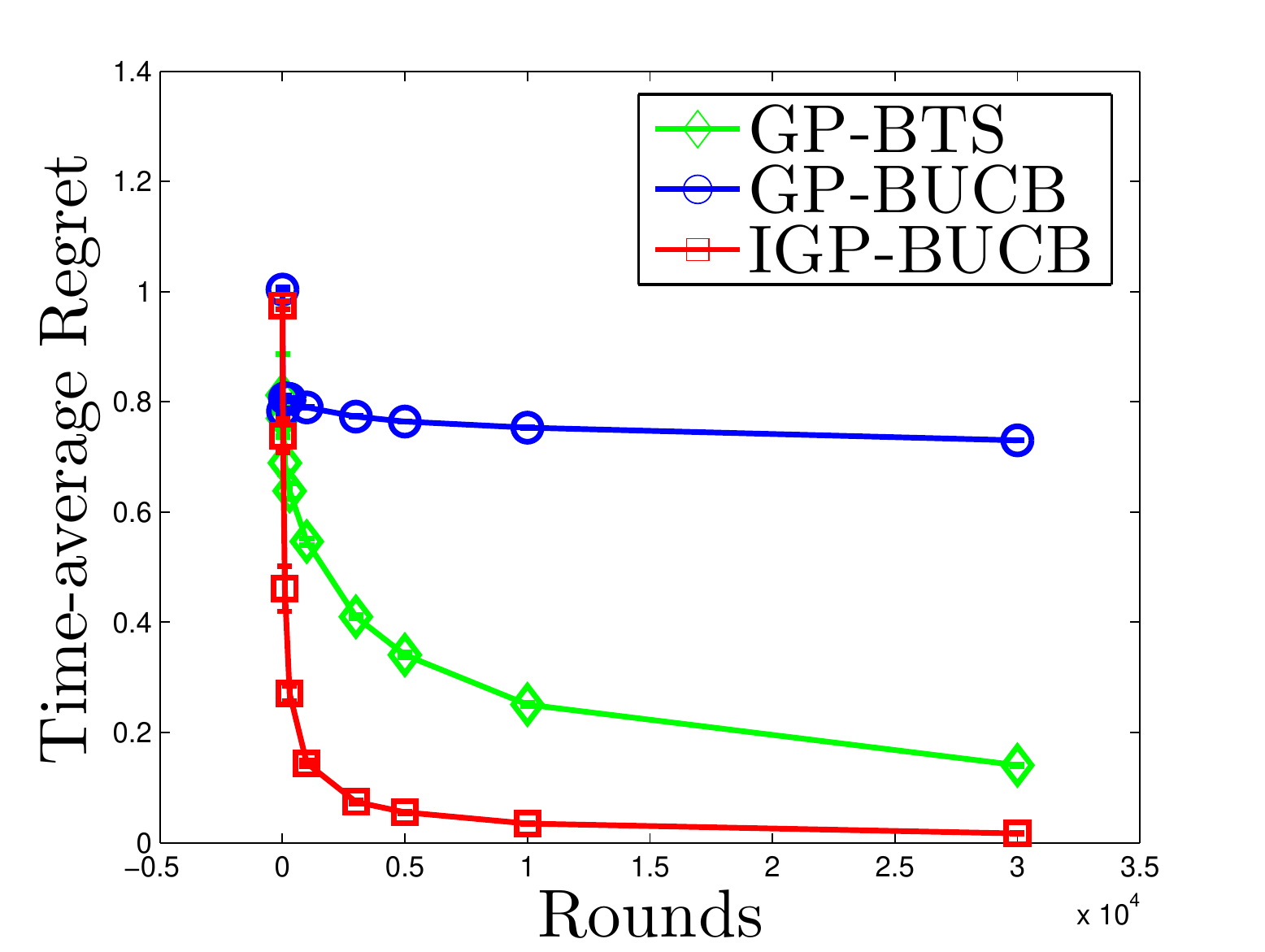}}
\subfigure[]{\includegraphics[height=1.15in,width=1.35in]{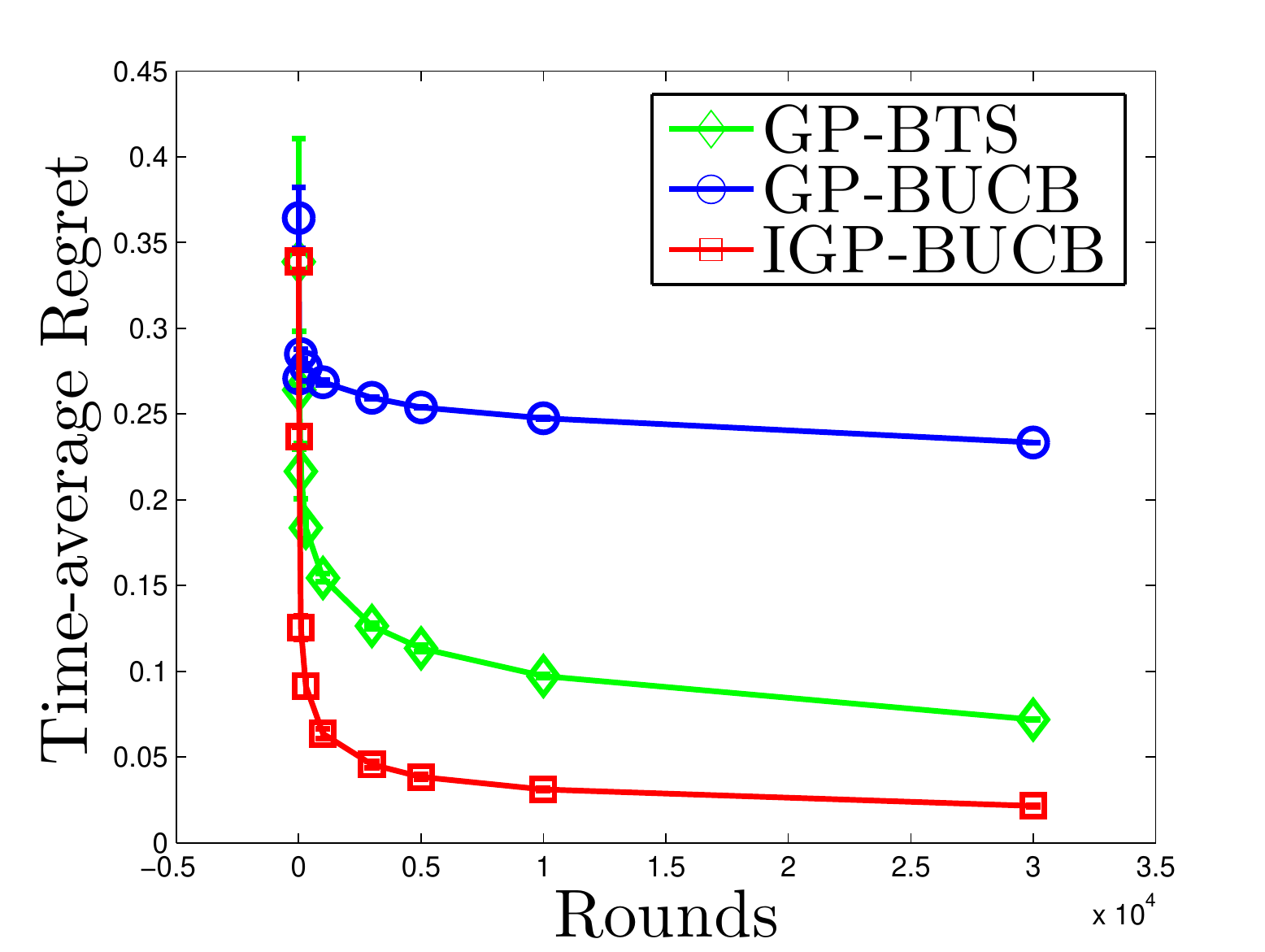}}
\subfigure[]{\includegraphics[height=1.15in,width=1.35in]{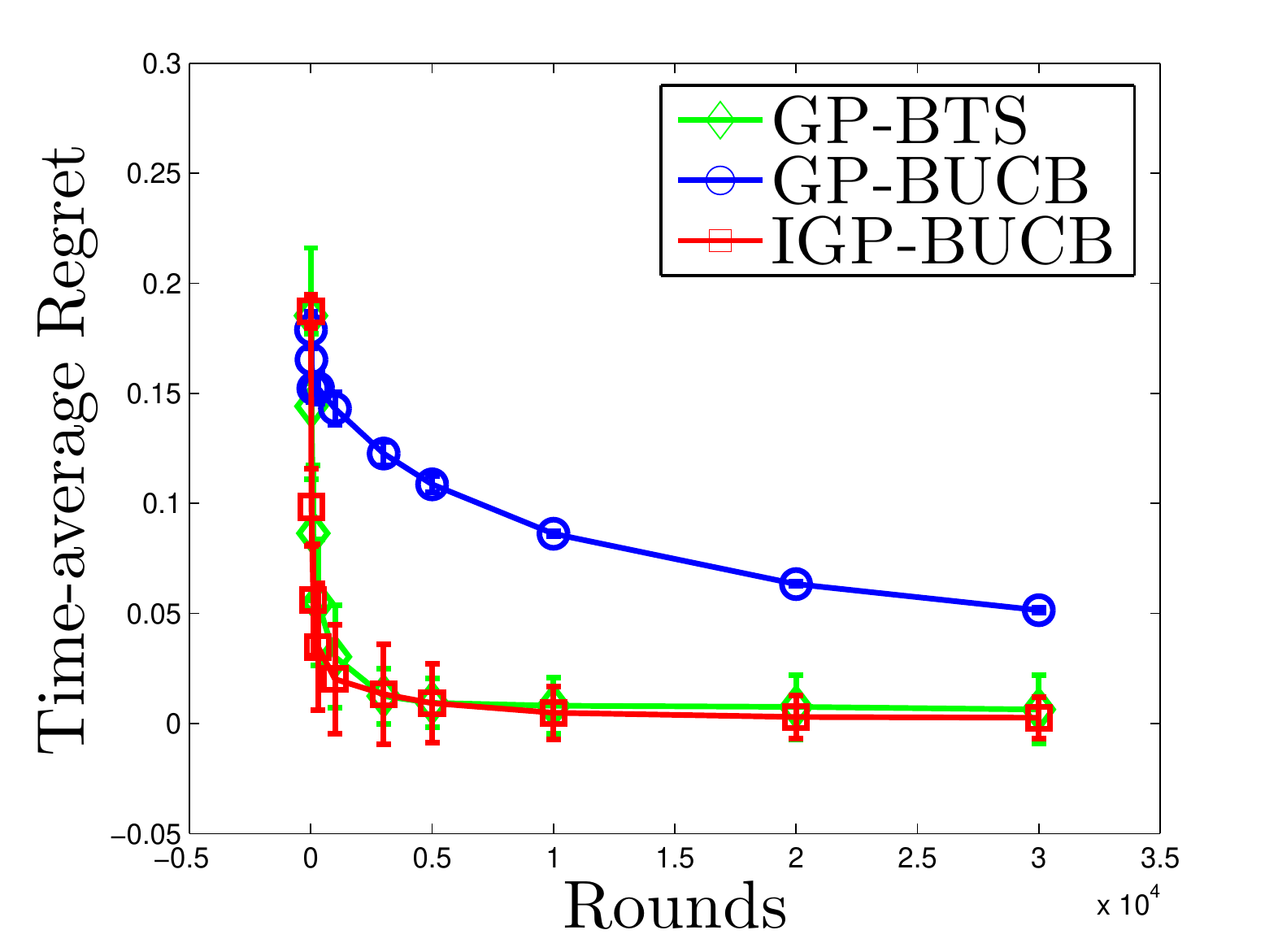}}
\subfigure[]{\includegraphics[height=1.15in,width=1.35in]{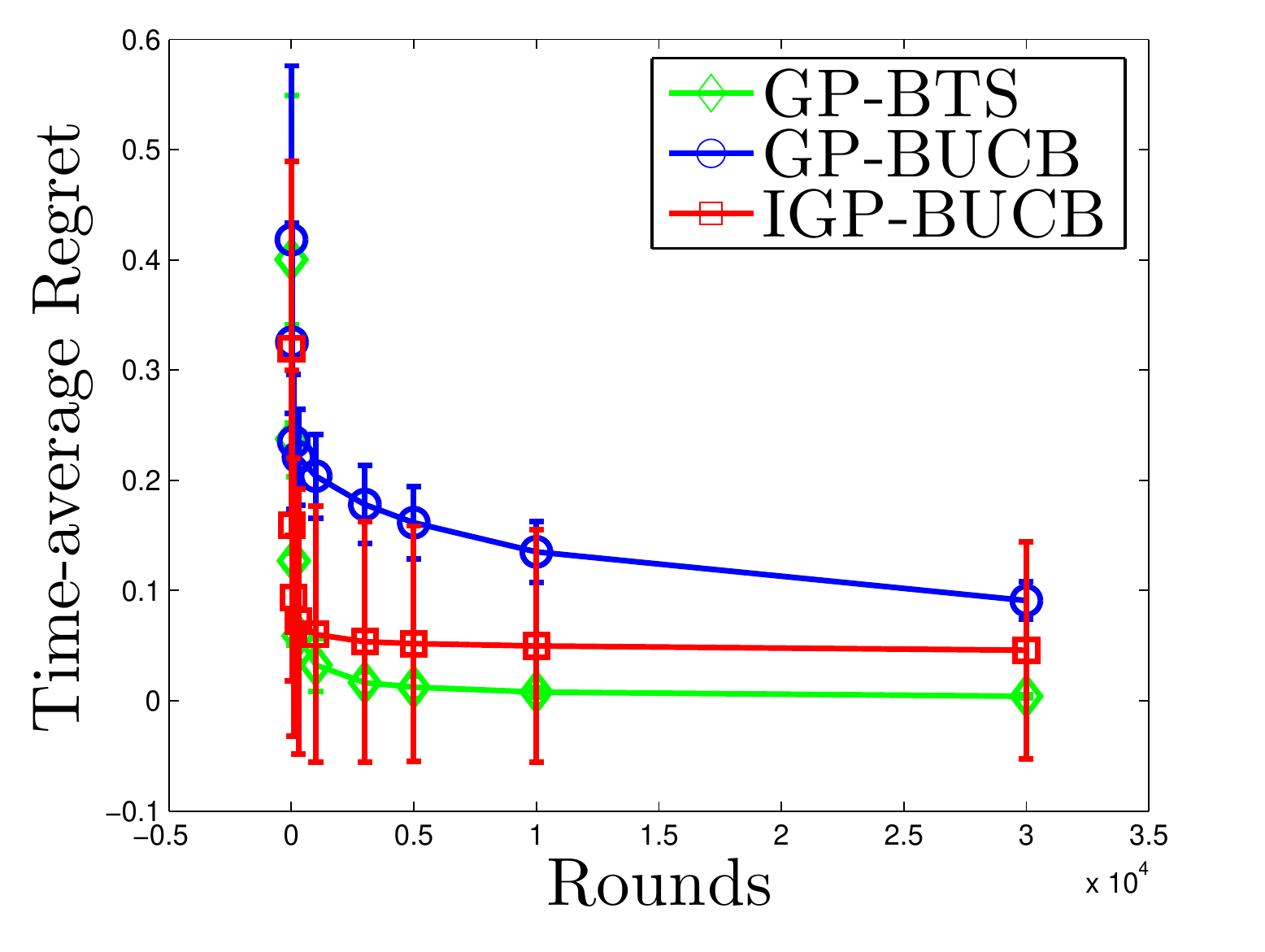}}
\caption{{\footnotesize Time-average regret for (a)  RKHS functions of SE kernel in simple batch setting, (b) RKHS functions of Mat$\acute{e}$rn kernel in simple batch setting, (c) RKHS functions of SE kernel in simple delay setting, (d) RKHS functions of Mat$\acute{e}$rn kernel in simple delay setting, (e) Cosines test function, (f) Rosenbrock test function, (g) temperature sensor data and (h) light sensor data.
}}
\label{fig:synthetic_plot_rkhs} 
\vskip -5mm
\end{figure}
\vspace*{-5mm}
\textbf{1. Functions from RKHS.} A set of $25$ functions is generated from RKHSs corresponding to the Mat$\acute{e}$rn and Squared-Exponential (SE) kernels with hyperparameters $l=0.2$, $\nu =2.5$, similar to the procedure of \citet{chowdhury2017kernelized}. $D$ is a discretization of $[0,1]$ into $100$ evenly spaced points. Comparison is done for both simple batch and simple delay settings with $M=5$.

\vspace*{-1mm}
\textbf{2. Benchmark functions.} We consider the Cosine and Rosenbrock test functions \cite{azimi2012hybrid}. $D$ is a $31 \times 31$ grid of evenly spaced points on $[0,1]^2$ and the kernel used is SE with $l^2=0.1$.

\vspace*{-1mm}
\textbf{3. Temperature\footnote{\url{http://db.csail.mit.edu/labdata/labdata.html}} and light sensor data\footnote{\url{http://www.cs.cmu.edu/~guestrin/Class/10708-F08/projects/lightsensor.zip}}.}  The algorithms are compared in the context of learning the maximum reading of the sensors \cite{srinivas2009gaussian}. The kernel used is the empirical covariance of the sensor readings, $\lambda$ is set to $5\%$ of the average empirical variance and $\gamma_t$ is set equal to $\ln t$.

\vspace*{-1mm}
\textbf{Observations:} IGP-BUCB outperforms GP-BUCB in all experiments, thus validating our theoretical bounds. For synthetic benchmarks, IGP-BUCB performs better than GP-BTS and for sensor data experiments, GP-BTS fares comparably, if not better, with IGP-BUCB.

\textbf{Challenges and future work.} The adaptive discretization in GP-BTS introduces an extra multiplicative factor in the regret bound. We believe the analysis can be done without resorting to the discretization and it remains an open problem even in the strictly sequential setting. From an applied point of view, there is  the important open question on how to efficiently and provably optimize the UCB rule or the functions randomly drawn from GPs.

\subsubsection*{Acknowledgments}
The authors are grateful to anonymous reviewers for providing useful comments. Sayak Ray Chowdhury is supported by Google India PhD fellowship.

\bibliographystyle{plainnat}
\bibliography{nips}

\newpage

\begin{appendix}
\huge 
\begin{center}
Appendix
\end{center}
\normalsize
\section{Relevant Definitions and Results}
\label{appendix:GP}
We first review some relevant definitions and results from the Gaussian process multi-armed bandits literature, which will be useful in the analysis of our algorithms. We first begin with the definition of \textit{Maximum Information Gain}, first appeared in \citet{srinivas2009gaussian}, which basically measures the reduction in uncertainty about the unknown function after some noisy observations (rewards).

For a function $f:D \ra \Real$ and any subset $A \subset D$ of its domain, we use $f_A := [f(x)]_{x\in A}$ to denote its restriction to $A$, i.e., a vector containing $f$'s evaluations at each point in $A$ (under an implicitly understood bijection from coordinates of the vector to points in $A$). In case $f$ is a random function, $f_A$ will be understood to be a random vector. For jointly distributed random variables $X, Y$, $I(X;Y)$ denotes the Shannon mutual information between them.

\begin{mydefinition}[Maximum Information Gain (MIG)] 
\label{def:mig}
Let $f:D\ra \Real$ be a (possibly random) real-valued function defined on a domain $\cX$, and $t$ a positive integer. For each subset $A \subset D$, let $Y_A$ denote a noisy version of $f_A$ obtained by passing $f_A$ through a channel $\prob{Y_A | f_A}$. The \textit{Maximum Information Gain (MIG)} about $f$ after $t$ noisy observations is defined as
\beqn
\gamma_t \bydef \max_{A \subset D : \abs{A}=t} I(f_A;Y_A).
\eeqn
(We omit mentioning explicitly the dependence on the channels for ease of notation.)
\end{mydefinition}

\textit{MIG} will serve as a key instrument to obtain our regret bounds by virtue of Lemma \ref{lem:sum-of-sd}.

For a kernel function $k: D\times D \ra \Real$ and points $x, x_1,\ldots,x_s \in D$, we define the vector $k_{s}(x)\bydef [k(x_1,x),\ldots,k(x_{s},x)]^T$ of kernel evaluations between $x$ and $x_1, \ldots, x_s$, and $K_{\{x_1, \ldots, x_s\}} \equiv K_{s} \bydef [k(x_i,x_j)]_{1 \le i,j \le s}$ be the kernel matrix induced by the $x_i$s. Also for each $x \in D$ and $\lambda > 0$, let  $\sigma_{s}^2(x) \bydef k(x,x) - k_{s}(x)^T(K_{s} + \lambda I)^{-1} k_{s}(x)$.

\begin{mylemma}[Information Gain and Predictive Variances under GP prior and additive Gaussian noise]
\label{lem:sum-of-sd}
Let $k: D \times D \ra \Real$ be a symmetric positive semi-definite kernel and $f \sim GP_{D}(0,k)$ a sample from the associated Gaussian process over $D$. For each subset $A \subset D$, let $Y_A$ denote a noisy version of $f_A$ obtained by passing $f_A$ through a channel that adds iid $\cN(0,\lambda)$ noise to each element of $f_A$. Then, \beq
\label{eqn:info-gain-zero}
\gamma_t  =  \max_{A \subset D : \abs{A}=t} \frac{1}{2} \ln \abs{I + \lambda^{-1}K_A},
\eeq
and
\beq
\label{eqn:info-gain-one}
\gamma_t = \max_{ \lbrace x_1,\ldots,x_t \rbrace \subset D}\frac{1}{2}\sum_{s=1}^{t}\ln\left(1 +\lambda^{-1}\sigma_{s-1}^2(x_s) \right).
\eeq
Further, if $k$ has bounded variance, i.e. $k(x,x) \le 1$ for all $x \in D$,
\beq
\label{eqn:info-gain-two}
\sum_{s=1}^{t}\sigma_{s-1}(x_s) \le \sqrt{t(2\lambda+1)\gamma_t}.
\eeq
\end{mylemma}
\begin{proof}
(\ref{eqn:info-gain-zero}) and (\ref{eqn:info-gain-one}) follow from \citet{srinivas2009gaussian}. 

Further from our assumption $k(x,x) \le 1$, we have $0 \le \sigma_{s-1}^2(x)\le 1$ for all $x \in D$, and hence $\sigma_{s-1}^2(x_s) \le \ln\big(1+\lambda^{-1}\sigma_{s-1}^2(x_s)\big)/\ln(1+\lambda^{-1})$ since $\alpha/\ln(1+\alpha)$ is non-decreasing for any $\alpha \in [0,\infty)$. Therefore
\beqn
\sum_{s=1}^{t}\sigma^2_{s-1}(x_s) \le 2/\ln(1+\lambda^{-1}) \sum_{s=1}^{t}\dfrac{1}{2}\ln\left(1+\lambda^{-1}\sigma_{s-1}^2(x_s)\right)\le 2\gamma_t/\ln(1+\lambda^{-1}),
\eeqn
where the last inequality follows from (\ref{eqn:info-gain-one}). Now see that $2/\ln(1+\lambda^{-1}) \le (2+\lambda^{-1})/\lambda^{-1} =2\lambda+1$, since $\ln(1+\alpha) \ge 2\alpha/(2+\alpha)$ for any $\alpha \in [0,\infty)$. Hence $\sum_{s=1}^{t}\sigma^2_{s-1}(x_s) \le (2\lambda +1)\gamma_t$. Now (\ref{eqn:info-gain-two}) follows from the Cauchy-Schwartz inequality: $\sum_{s=1}^{t}\sigma_{s-1}(x_s) \le \sqrt{t\sum_{s=1}^{t}\sigma^2_{s-1}(x_s) }$.
\end{proof}

\paragraph{Bound on Maximum Information Gain} Note that the right hand sides of (\ref{eqn:info-gain-zero}) and (\ref{eqn:info-gain-one}) depend only on the kernel function $k$, domain $D$, and number of observations $t$.
\citet{srinivas2009gaussian} proved upper bounds over $\gamma_t$ for three commonly used kernels, namely \textit{Linear}, \textit{Squared Exponential} and \textit{Mat$\acute{e}$rn}, defined respectively as
\beqan
k_{Linear}(x,x') &=& x^Tx',\\
k_{SE}(x,x') &=& \exp\left(-s^2/2l^2\right), \\
k_{Mat\acute{e}rn}(x,x') &=& \frac{2^{1-\nu}}{\Gamma(\nu)}\left(\frac{s\sqrt{2\nu}}{l}\right)^\nu B_\nu\left(\frac{s\sqrt{2\nu}}{l}\right),
\eeqan
where $l > 0$ and $\nu > 0$ are hyper-parameters of the kernels, $s = \norm{x-x'}_2$ encodes the similarity between two points $x,x'\in D$ and $B_\nu$ denotes the \textit{modified Bessel function}. The bounds are given in Lemma \ref{lem:info-gain-bound}.
\begin{mylemma}[MIG for common kernels]
\label{lem:info-gain-bound}
Let $k: D \times D \ra \Real$ be a symmetric positive semi-definite kernel and $f \sim GP_{D}(0,k)$. Let $D$ be a compact and convex subset of $\Real^d$ and the kernel $k$ satisfies $k(x,x') \le 1$ for all $x,x' \in D$. Then for
\begin{itemize}
\item Linear kernel: $\gamma_t=O(d\ln t)$.
\item Squared Exponential kernel: $\gamma_t=O\left((\ln t)^{d}\right)$.
\item Mat$\acute{e}$rn kernel: $\gamma_t=O\left(t^{d(d+1)/(2\nu+d(d+1))}\ln t\right)$.
\end{itemize}
\end{mylemma}

Note that, the Maximum Information Gain $\gamma_t$ depends only \textit{sublinearly} on the number of observations $t$ for all these kernels.

Lemma \ref{lem:true-function-bound} (appeared independently in \citet{pmlr-v70-chowdhury17a} and 
\citet{durand2017streaming}.) gives a concentration bound for a member $f$ of the RKHS $\cH_k(D)$. 

\begin{mylemma}[Concentration of an RKHS member]
\label{lem:true-function-bound}
Let $k: D \times D \to \Real$ be a symmetric, positive-semidefinite kernel and $f:D \to \Real$ be a member of the RKHS $\cH_k(D)$ of real-valued functions on $D$ with kernel $k$. Let $\lbrace x_t \rbrace_{t \ge 1}$ and $\lbrace \epsilon_t \rbrace_{t \ge 1}$ be stochastic processes such that $\lbrace x_t \rbrace_{t \ge 1}$ form a predictable process, i.e., $x_t \in \sigma(\lbrace x_s, \epsilon_s\rbrace_{s = 1}^{t-1})$ for each $t$, and $\lbrace \epsilon_t \rbrace_{t \ge 1}$ is conditionally $R$-sub-Gaussian for a positive constant $R$, i.e.,
\beqn
\forall t \ge 0,\;\;\forall \lambda \in \Real, \;\; \expect{e^{\lambda \epsilon_t} \given \cF_{t-1}} \le \exp\left(\frac{\lambda^2R^2}{2}\right),
\eeqn
where $\cF_{t-1}$ is the $\sigma$-algebra generated by $\lbrace x_s, \epsilon_s\rbrace_{s = 1}^{t-1}$ and $x_t$. Let $\lbrace y_t\rbrace_{t\ge 1}$ be a sequence of noisy observations at the query points $\lbrace x_t \rbrace_{t \ge 1}$, where $y_t=f(x_t)+\epsilon_t$. For $\lambda > 0$ and $x \in D$, let
\beqan
\mu_{t-1}(x)&:=& k_{t-1}(x)^T(K_{t-1} + \lambda I)^{-1}Y_{t-1},\\
\sigma_{t-1}^2(x)&:=&k(x,x) - k_{t-1}(x)^T(K_{t-1} + \lambda I)^{-1} k_{t-1}(x),
\eeqan
where $Y_{t-1}\bydef[y_1,\ldots,y_{t-1}]^T$ denotes the vector of observations at $\lbrace x_1,\ldots,x_{t-1}\rbrace$.
Then, for any $0 < \delta \le 1$, with probability at least $1-\delta$, uniformly over $t \ge 1, x \in D$,
\beqn
\abs{f(x)-\mu_{t-1}(x)}\le \Big(\norm{f}_k + \dfrac{R}{\sqrt{\lambda}}\sqrt{2\big(\ln(1/\delta)+ \gamma_{t-1} \big)}\Big)\sigma_{t-1}(x),
\eeqn
where $\gamma_t$ is the Maximum Information Gain about any $f \sim GP_D(0,k)$ after $t$ noisy observations obtained by passing $f$ through an iid Gaussian channel $\cN(0,\lambda)$.
\end{mylemma}
\begin{proof}
The proof follows from the proof of Theorem 2.1 in \citet{durand2017streaming}.
\end{proof}

Now we define a quantity $\xi_M$ (modified from \citet{kandasamy2018parallelised}), which essentially measures the information about $f$ that gets hallucinated each round due to at most $M-1$ hallucinated observations, conditioned on the actual observations.
\begin{mydefinition}[Maximum Hallucinated Information]
\label{def:hallucinated-info}
Let $f:D\ra \Real$, $\cS(t) : \mathbb{N} \rightarrow \mathbb{Z}_{+}$ be a mapping
such that $\mathcal{S}(t)  \le t-1$ for all $t \ge 1$ and $M$ be a constant such that $t-\cS(t) \le M$ for all $t \ge 1$. Then, $\xi_M$ denotes the maximum hallucinated information about $f$ due to $t-\cS(t)-1$ hallucinated observations (there are at most $M-1$ of them at every $t$) in the sense that, for all $x \in D$,
\beqn
I\big(f(x);Y_{\cS(t)+1:t-1}\given Y_{1:\cS(t)}\big) \le 1/2 \ln(\xi_M),
\eeqn
where $Y_{1:\cS(t)}$ is the vector of actual observations up to round $t$, $Y_{\cS(t)+1:t-1}$ is the vector of hallucinated observations and $I\left(f(x);Y_{\cS(t)+1:t-1}\given Y_{1:\cS(t)}\right)$ is the conditional mutual information between $f(x)$ and $Y_{\cS(t)+1:t-1}$, given $Y_{1:\cS(t)}$.
\end{mydefinition}
The following result is modified from \citet{desautels2014parallelizing}, and provide a choice of $\xi_M$.
\begin{mylemma}[Relation between the Maximum Information Gain and Hallucinated Information]
Let $f:D\ra \Real$ be a function, $\cS(t) : \mathbb{N} \rightarrow \mathbb{Z}_{+}$ be a mapping
such that $\mathcal{S}(t)  \le t-1$ for all $t \ge 1$ and $M$ be a constant such that $t-\cS(t) \le M$ for all $t \ge 1$. Further, let $\gamma_t$ be the \textit{Maximum Information Gain} about $f$ after $t$ observations (Definition \ref{def:mig}) and $\xi_M$ be the maximum hallucinated information (Definition \ref{def:hallucinated-info}). Then, $\xi_M = \exp(2\gamma_{M-1})$.
\end{mylemma}
The proof follows from the fact that $I\left(f(x);Y_{\cS(t)+1:t-1}\given Y_{1:\cS(t)}\right) \le \gamma_{M-1}$ for all $x \in D$  \cite{desautels2014parallelizing}.

The next lemma is due to \citet[Proposition 1]{desautels2014parallelizing} and is pivotal in the analysis of batch Bayesian optimization.
\begin{mylemma}[Ratio of Posterior standard deviations bounded by Hallucinated Information]
\label{lem:conditional-mi}
Let $k: D \times D \to \Real$ be a symmetric, positive-semidefinite kernel and $f \sim GP_D(0, k)$. Let $\sigma_{\cS(t)}$ and $\sigma_{t-1}$ be the posterior standard deviations, respectively conditioned on first $\cS(t)$ and $t-1$ queries. Then, for all $x \in D$,
\beqn
\dfrac{\sigma_{\cS(t)}(x)}{\sigma_{t-1}(x)}=\exp \Big(I\big(f(x);Y_{\cS(t)+1:t-1}\given Y_{1:\cS(t)}\big)\Big).
\eeqn
\end{mylemma}
Definition \ref{def:hallucinated-info}, along with Lemma \ref{lem:conditional-mi}, implies that $\dfrac{\sigma_{\cS(t)}(x)}{\sigma_{t-1}(x)}) \le \xi_M^{1/2}$ for all $x \in D$. 



\section{Regret Analysis for IGP-BUCB Algorithm}
\label{appendix:UCB}
First we begin with the following lemma, which states that the reward function $f$ is always well concentrated within properly constructed confidence intervals in the batch setting.
\begin{mylemma}[Concentration of reward function in the batch setting]
\label{lem:batch-concentration}
Let $k: D \times D \to \Real$ be a symmetric, positive-semidefinite kernel and $f:D \to \Real$ be a member of the RKHS $\cH_k(D)$ of real-valued functions on $D$ corresponding to kernel $k$, with RKHS norm bounded by $B$. Further, let $\lbrace y_t\rbrace_{t\ge 1}$ be a sequence of noisy observations at queries $\lbrace x_t \rbrace_{t \ge 1}$, where $y_t=f(x_t)+\epsilon_t$ and the noise sequence $\lbrace \epsilon_t \rbrace_{t \ge 1}$ be conditionally $R$-sub-Gaussian. Let $\cS(t) : \mathbb{N} \rightarrow \mathbb{Z}_{+}$ be a mapping
such that $\mathcal{S}(t)  \le t-1$ for all $t \ge 1$, 
$\mu_{\cS(t)}$ be the posterior mean and $\sigma_{t-1}^2$ be the posterior variance, after $\cS(t)$ and $t-1$ rounds, respectively. Then, for any $0 < \delta \le 1$, the following holds:
\beqn
\prob{\forall t \ge 1,\; \forall x \in D, \; \abs{f(x)-\mu_{\cS(t)}(x)}\le \xi_M^{1/2}\left(B + \dfrac{R}{\sqrt{\lambda}}\sqrt{2\big(\gamma_{\cS(t)}+ \ln(1/\delta)\big)}\right)\sigma_{t-1}(x)} \ge 1-\delta.
\eeqn
\end{mylemma}
\begin{proof}
Recall that, for any $0 < \delta \le 1$, the decision rule of IGP-BUCB algorithm is 
\beqn
x_t=\argmax_{x\in D}\mu_{\cS(t)}(x)+\beta_t\sigma_{t-1}(x),
\eeqn
where $\beta_t=\xi_M^{1/2}\left(B + \dfrac{R}{\sqrt{\lambda}}\sqrt{2\big(\gamma_{\cS(t)}+ \ln(1/\delta)\big)}\right)$. Implicit in this decision rule is the corresponding confidence interval
for each $t \ge 1$ and for each $x\in D$, 
\beqn
C_t^{batch}(x) = \left[\mu_{\cS(t)}(x)-\beta_t \sigma_{t-1}(x),\mu_{\cS(t)}(x)+\beta_t \sigma_{t-1}(x)\right].
\eeqn
A special case of the batch setup is the strictly sequential setup with $\cS_t=t-1, M=1$ and $\xi_M=1$. Here, the confidence intervals take the form
\beqn
C_t^{seq}(x) = \left[\mu_{t-1}(x)-\alpha_t \sigma_{t-1}(x),\mu_{t-1}(x)+\alpha_t \sigma_{t-1}(x)\right],
\eeqn
where $\alpha_t = \left(B + \dfrac{R}{\sqrt{\lambda}}\sqrt{2\big(\gamma_{t-1}+ \ln(1/\delta)\big)}\right)$. Thus, Lemma \ref{lem:true-function-bound} implies that
\beqn
\prob{\forall t \ge 1,\; \forall x \in D, \; \abs{f(x)-\mu_{t-1}(x)}\le \alpha_t\sigma_{t-1}(x)} \ge 1-\delta,
\eeqn
and therefore
\beq
\label{eqn:seq-conc}
\prob{\forall t \ge 1, \forall x \in D, f(x) \in C_t^{seq}(x)} \ge 1-\delta.
\eeq
Now, consider the confidence interval
\beqn
C_{\cS(t)+1}^{seq}(x) = \left[\mu_{\cS(t)}(x)-\alpha_{\cS(t)+1} \sigma_{\cS(t)}(x),\mu_{\cS(t)}(x)+\alpha_{\cS(t)+1} \sigma_{\cS(t)}(x)\right].
\eeqn
Observe that both the intervals $C_t^{batch}(x)$ and $C_{\cS(t)+1}^{seq}(x)$ are centered around $\mu_{\cS(t)}(x)$, and their widths are $2\beta_t\sigma_{t-1}(x)$ and $2\alpha_{\cS(t)+1} \sigma_{\cS(t)}(x)$, respectively.
Further, see that $\beta_t=\xi_M^{1/2}\alpha_{\cS(t)+1}$. This, along with the fact that $\dfrac{\sigma_{\cS(t)}(x)}{\sigma_{t-1}(x)}\le \xi_M^{1/2}$, implies 
\beqn
\alpha_{\cS(t)+1} \sigma_{\cS(t)}(x) \le \xi_M^{1/2}\alpha_{\cS(t)+1}\sigma_{t-1}(x)=\beta_t\sigma_{t-1}(x).
\eeqn
Thus, we have $C_{\cS(t)+1}^{seq}(x) \subseteq C_t^{batch}(x)$ for all $x \in D$ and for all $t \ge 1$. Therefore,
\beq
\label{eqn:seq-batch-relation}
f(x) \in C_{\cS(t)+1}^{seq}(x), \forall x \in D, \forall t \ge 1 \Longrightarrow f(x) \in C_{t}^{batch}(x), \forall x \in D, \forall t \ge 1.
\eeq
Also, as $0 \le \cS(t) \le t-1$ for every $t \ge 1$, we have
\beq
\label{eqn:seq-seq-relation}
f(x) \in C_t^{seq}(x), \forall t \ge 1, \forall x \in D \Longrightarrow f(x) \in C_{\cS(t)+1}^{seq}(x), \forall x \in D, \forall t \ge 1.
\eeq
Now combining equations \ref{eqn:seq-conc}, \ref{eqn:seq-batch-relation} and \ref{eqn:seq-seq-relation}, we get
\beqn
\prob{\forall t \ge 1, \forall x \in D, f(x) \in C_t^{batch}(x)} \ge 1-\delta.
\eeqn
Finally, the result follows from the definition of the confidence interval $C_t^{batch}(x)$.
\end{proof}
\subsection{Proof of Theorem \ref{thm:regret-bound-UCB-approx}}
For every round $t \ge 1$, the decision rule of IGP-BUCB (Algorithm \ref{algo:ucb}) implies that
\beqn
\mu_{\cS(t)}(x_t)+\beta_t\sigma_{t-1}(x_t) \ge \mu_{\cS(t)}(x^\star)+\beta_t\sigma_{t-1}(x^\star),
\eeqn
where $\beta_t = \xi_M^{1/2}\left(B + \dfrac{R}{\sqrt{\lambda}}\sqrt{2\big(\gamma_{\cS(t)}+\ln(1/\delta)\big)}\right)$. From Lemma \ref{lem:batch-concentration}, with probability at least $1-\delta$, 
\beqn
f(x^\star) \le \mu_{\cS(t)}(x^\star)+\beta_t\sigma_{t-1}(x^\star) \quad \text{and} \quad \mu_{\cS(t)}(x_t)-f(x_t) \le \beta_t\sigma_{t-1}(x_t) \quad \text{for all}\quad t\ge 1.
\eeqn
Therefore, with probability at least $1-\delta$, we have for all $t \ge 1$, the instantaneous regret
\beqan
r_t = f(x^\star) - f(x_t)
\le \mu_{\cS(t)}(x_t)+\beta_t\sigma_{t-1}(x_t)-f(x_t)
\le 2\beta_t\sigma_{t-1}(x_t).
\eeqan
Hence, with probability at least $1-\delta$, the cumulative regret $R_T=\sum\limits_{t=1}^{T}r_t \le 2\sum\limits_{t=1}^{T}\beta_t\sigma_{t-1}(x_t)$. Now from Definition \ref{def:mig}, see that $\gamma_{\cS(t)}$ doesn't decrease with $t$. Hence, $\gamma_{\cS(t)} \le \gamma_{\cS(T)}$ and thus $\beta_t \le \beta_T$ for all $t \le T$. Therefore, with probability at least $1-\delta$, $R_T \le  2\beta_T\sum\limits_{t=1}^{T}\sigma_{t-1}(x_t)$. Further, $\beta_T \le \xi_M^{1/2}\left(B + R\sqrt{2\big(\gamma_T+\ln(1/\delta)\big)}\right)$, since $\cS(T) \le T-1 < T$. From Lemma \ref{lem:sum-of-sd}, $\sum\limits_{t=1}^{T}\sigma_{t-1}(x_t) \le \sqrt{(2\lambda+1)T\gamma_T}$. Hence, with probability at least $1-\delta$,
\beqan
R_T &\le& 2 \sqrt{\xi_M}\left(B + R\sqrt{2\big(\gamma_T+\ln(1/\delta)\big)}\right)\sqrt{(2\lambda+1)T\gamma_T}\\ &=& O\left(B\sqrt{\xi_M T\gamma_T}+\sqrt{\xi_M T\gamma_T\big(\gamma_T+\ln(1/\delta)\big)}\right).
\eeqan
Thus $R_T = O\Big(\sqrt{\xi_M T}\big(B\sqrt{\gamma_T}+\gamma_T\big)\Big)$ with high probability.

\section{ Regret Analysis for GP-BTS Algorithm}
\label{appendix:TS}
\subsection{Useful Lemmas and Definitions}
\begin{mylemma}[Lipschitzness of RKHS functions]
\label{lem:lipschitz}
Let $f:D\ra \Real$ be a function in the RKHS $H_k(D)$ with $\norm{f}_k \le B$ and $D\subset \Real^d$. Let the kernel $k$ be continuously differentiable and let $L$ be a constant satisfying $L^2=\sup\limits_{x\in D}\left(\partial_p\partial_qk(p,q)\given_{p=q=x}\right)$. Then, for all $x,y \in D$,
\beqn
\abs{f(x)-f(y)} \le B L \norm{x-y}_1.
\eeqn
\end{mylemma}
\begin{proof}
Fix $x,y \in D$. Now, by using reproducing property and Cauchy-Schwartz inequality, we have
\beqn
\abs{f(x)-f(y)} = \abs{\inner{f}{k(x,\cdot)-k(y,\cdot)}_k} \le \norm{f}_k\norm{k(x,\cdot)-k(y,\cdot)}_k = \norm{f}_k \sqrt{k(x,x)-2k(x,y)+k(y,y)}.
\eeqn
Now since $k$ is continuously differentiable, by the choice of $L$, we have 
\beqn
k(x,x)-2k(x,y)+k(y,y) \le L^2\norm{x-y}^2_2.
\eeqn
Now, the result follows from the fact that $\norm{\cdot}_2 \le \norm{\cdot}_1$ and $\norm{f}_k \le B$.
\end{proof}
\begin{mylemma}[Gaussian anti-concentration inequality \cite{abramowitz1966handbook}]
\label{lem:anti-concentration}
Let $X \sim \cN(\mu,\sigma^2)$. Then, for any $\theta > 0$,
\beqn
\prob{\dfrac{X-\mu}{\sigma} > \theta} \ge \dfrac{e^{-\theta^2}}{4\sqrt{\pi}\theta}.
\eeqn
\end{mylemma}

\textbf{Choice of discretization.} At each round $t$, the decision set used by GP-BTS is restricted to be a {unique} discretization $D_t$ of $D$ with the property that $\abs{f(x)-f([x]_t)} \le 1/t^2$ for all $x \in D$, where $[x]_t \bydef \argmin_{x' \in D_t} \norm{x-x'}_1$ denotes the closest point to $x$ in $D_t$ in the sense of $\norm{\cdot}_1$-norm. 

Note that, this can be achieved by restricting the domain $D$ to be a compact and convex subset of $[0,r]^d$ for some constant $r \ge 0$ and by choosing the discretization sets $D_t $ such that every coordinate has $BLrdt^2$ uniformly spaced points.
This implies for all $x \in D$, $\norm{x-[x]_t}_1 \le rd/BLrdt^2 = 1/BLt^2$, and hence from Lemma \ref{lem:lipschitz}, $\abs{f(x)-f([x]_t)} \le B L \norm{x-[x]_t}_1 \le 1/t^2$.



%
%
\begin{mydefinition}
For each $t \ge 1$, let $\cH_{t-1} \bydef \lbrace x_1,\ldots,x_{t-1},y_1,\ldots,y_{\cS(t)}\rbrace $ be the set of all queries and observations available before the start of round $t$. By definition, $\cH_1\subseteq \cH_2 \subseteq \cdots$, and thus the sequence $\lbrace \cH_t\rbrace_{t\ge 0}$ defines a filtration. 
\label{def:filtration}
\end{mydefinition}
\begin{mydefinition}
\label{def:two-events}

For each $t \ge 1$, define the events $E_f^{batch}(t)$ and $E_{f_t}^{batch}(t)$ as
\beqan
E_f^{batch}(t) &\bydef & \left\lbrace \forall x \in D, \abs{f(x)-\mu_{\cS(t)}(x)} \le v_t\; \sigma_{t-1}(x) \right\rbrace,\\
E_{f_t}^{batch}(t) &\bydef & \left\lbrace \forall x \in D_t, 
\abs{f_t(x)-\mu_{\cS(t)}(x)} \le v_t w_t \;\sigma_{t-1}(x)
\right\rbrace,
\eeqan
where $v_t \bydef \sqrt{\xi_M}\Big(B + \dfrac{R}{\sqrt{\lambda}}\sqrt{2\left(\gamma_{\cS(t)}+\ln(2/\delta)\right)}\Big)$, ${w}_t \bydef \sqrt{4\ln t + 2d\ln(BLrdt^2)}$ and $c_t \bydef v_t(1+w_t)$.
\end{mydefinition}
From Definition $\ref{def:mig}$, see that $\gamma_{\cS(t)}$ is non-decreasing with $t$. This in turn implies that $v_t$, as well as $c_t$ doesn't decrease with $t$. Also see that, the event $E_f^{batch}(t)$ is completely deterministic given $\cH_{t-1}$.
\begin{mydefinition}
\label{def:saturated-arms}
For all $x \in D$, let $\Delta_t(x) \bydef f([x^\star]_t)-f(x)$ be the difference between the values of the reward function at $[x^\star]_t$ and at $x$. Then, the set of saturated points $S_t$ within the discretization set $D_t$ is defined as
\beqn
S_t \bydef \left \lbrace x\in D_t: \Delta_t(x) > c_t\sigma_{t-1}(x)\right \rbrace.
\eeqn
\end{mydefinition}
Note that $\Delta_t([x^\star]_t) = 0$ for all $t \ge 1$, and hence $[x^\star]_t\in D_t$ is unsaturated. 

Now, we prove the following lemmas, which will be useful in our analysis. The first two lemmas imply that the events $E_f^{batch}(t)$ and $E_{f_t}^{batch}(t)$ occur with high probability.
\begin{mylemma}
\label{lem:true-concentration}
Let $E_f^{batch}(t)$ be the event as in Definition \ref{def:two-events}. Then, for any $0 < \delta \le 1$,
$\prob{\forall t \ge 1, E_f^{batch}(t)}\ge 1-\delta/2 $.
\end{mylemma}
\begin{proof}
The proof follows from Lemma \ref{lem:batch-concentration}, using $\delta=\frac{\delta}{2}$.
\end{proof}
\begin{mylemma}
Let  $E_{f_t}^{batch}(t)$ be the event as in Definition \ref{def:two-events}, and $\lbrace \cH_t\rbrace_{t\ge 0}$ be the filtration as in Definition \ref{def:filtration}. Then, for all possible filtrations $\cH_{t-1}$, $\prob{E_{f_t}^{batch}(t)\given \cH_{t-1}} \ge 1-1/t^2$.
\label{lem:sampled-concentration}
\end{mylemma}
\begin{proof}
Fix any $t \ge 1$. Now at every round $t$, GP-BTS samples $f_t$ from $GP_{D_t}\left(\mu_{\cS(t)},v_t^2k_{t-1}\right)$, i.e. $f_t(x)\given \cH_{t-1} \sim \cN \left(\mu_{\cS(t)}(x),v_t^2\sigma^2_{t-1}(x)\right)$ for all $x \in D_t$. If $a \sim \cN(0,1)$, $c \ge 0$, then $\prob{\abs{a}\ge c} \le \exp(-c^2/2)$. Using this Gaussian concentration inequality, for any $0 < \delta \le 1$,
\beqn
\prob{\abs{f_t(x)-\mu_{t-1}(x)} \le \sqrt{2\ln(1/\delta)}\;v_t\sigma_{t-1}(x) \given \cH_{t-1}} \ge 1-\delta.
\eeqn
Now applying union bound over all $x \in D_t$, we have
\beqn
\prob{\forall x \in D_t, \abs{f_t(x)-\mu_{t-1}(x)} \le v_t\sqrt{2\ln(\abs{D_t}/\delta)}\;\sigma_{t-1}(x) \given \cH_{t-1}} \ge 1-\delta.
\eeqn
Here, $\abs{D_t}$ denotes the size of the discretization set $D_t$.
Now the result follows by using $\delta = 1/t^2$ and $\abs{D_t}=(BLrdt^2)^d$.
\end{proof}
The next lemma bound the probability of the bad event, i.e. the event that value of the sampled function is greater than the original function.
\begin{mylemma}
\label{lem:compare-sample-with-original}
Let $\cH_{t-1}$ be any filtration such that $E_f^{batch}(t)$ is true, and let $p = \dfrac{1}{4e\sqrt{\pi}}$. Then for all $x \in D$,
\beqn
\prob{f_t(x)>f(x)\given \cH_{t-1}} \ge p.
\eeqn
\end{mylemma}
\begin{proof}
Fix any $x\in D$. Now, since $\cH_{t-1}$ be any filtration such that the event $E_f^{batch}(t)$ is true, we have $\abs{f(x)-\mu_{\cS(t)(x)}} \le v_t\sigma_{t-1}(x)$. Further see that $f_t(x)\given \cH_{t-1} \sim \cN \left(\mu_{\cS(t)}(x),v_t^2\sigma^2_{t-1}(x)\right)$, and hence $\dfrac{f_t(x)-\mu_{\cS(t)}(x)}{v_t \sigma_{t-1}(x)}\given \cH_{t-1} \sim \cN(0,1)$. Then from Lemma \ref{lem:anti-concentration}, we have
\beqan
\prob{f_t(x)>f(x)\given \cH_{t-1}} &=& \prob{\dfrac{f_t(x)-\mu_{\cS(t)}(x)}{v_t \sigma_{t-1}(x)}>\dfrac{f(x)-\mu_{\cS(t)}(x)}{v_t \sigma_{t-1}(x)}\given \cH_{t-1}}\\
&\ge&\prob{\dfrac{f_t(x)-\mu_{\cS(t)}(x)}{v_t \sigma_{t-1}(x)}>\dfrac{\abs{f(x)-\mu_{\cS(t)}(x)}}{v_t \sigma_{t-1}(x)}\given \cH_{t-1}}\\
&\ge&\dfrac{1}{4\sqrt{\pi}\theta_t}e^{-\theta_t^2},
\eeqan
where $\theta_t = \frac{\abs{f(x)-\mu_{\cS(t)}(x)}}{v_t \sigma_{t-1}(x)} \le 1$. This implies that $\prob{f_t(x)>f(x)\given \cH_{t-1}} \ge \dfrac{1}{4e\sqrt{\pi}} =p$.
\end{proof}
\subsection{Proof of Theorem \ref{thm:regret-bound-TS-approx}}
Fix any $0 < \delta \le 1$. Then, following the techniques of \citet[Lemma 9,10,12 and 13]{chowdhury2017kernelized}, we can show that, with probability at least $1-\delta$,
\beq
\label{eqn:cum-regret-ts}
R_T=\dfrac{11c_T}{p}\sum_{t=1}^{T}\sigma_{t-1}(x_t)  +\dfrac{(2B+1)\pi^2}{6}+\dfrac{(4B+11)c_T}{p}\sqrt{2T\ln(2/\delta)}.
\eeq
Now from Lemma \ref{lem:sum-of-sd},  $\sum\limits_{t=1}^{T}\sigma_{t-1}(x_t) \le \sqrt{(2\lambda+1)T\gamma_T}$. Also $v_T \le \sqrt{\xi_M}\Big(B + \dfrac{R}{\sqrt{\lambda}}\sqrt{2\left(\gamma_T+\ln(2/\delta)\right)}\Big)$, since $\cS(T) \le T-1 < T$. Hence, from Definition \ref{def:two-events},
\beqn
C_T \le  \sqrt{\xi_M}\Big(B + \dfrac{R}{\sqrt{\lambda}}\sqrt{2\left(\gamma_T+\ln(2/\delta)\right)}\Big) \left(1+ \sqrt{4\ln T + 2d\ln(BLrdT^2)}\right),
\eeqn
and thus $C_T=O\Big(\sqrt{\xi_M \big(\gamma_T+\ln(2/\delta)\big)d\ln(BdT)}\Big)$.
Therefore, with probability at least $1-\delta$,
\beqan
R_T=O\Bigg(\sqrt{\xi_M \big(\gamma_T+\ln(2/\delta)\big)d\ln(BdT)} \left(\sqrt{T\gamma_T}+B\sqrt{T\ln(2/\delta)}\right)\Bigg).
\eeqan
Thus $R_T = O\left(\sqrt{\xi_M T\gamma_T^2 d\ln (BdT)} + B\sqrt{\xi_MT\gamma_Td\ln (BdT)}\right)= O\Big(\sqrt{\xi_M Td\ln (BdT)}\left(B\sqrt{\gamma_T}+\gamma_T\right)\Big)$ with high probability.

\section{Description of the Initialization Scheme}
\label{appendix:init}

First an initial batch of size $T^{init}$ is selected by greedily choosing points based on the posterior variances, i.e. by choosing $x_t=\argmax_{x\in D}\sigma_{t-1}(x)$ for all $t=1,\ldots,T^{init}$. Then, rewards are sampled for this set of initial points and the posterior GP is updated, which is used as the prior GP in our algorithms. \citet{desautels2014parallelizing} show that by choosing $T^{init}$ appropriately for different kernels\footnote{We refer the reader to \citet[Table 1]{desautels2014parallelizing} for values of $C_k$ and
$T^{init}$ for common kernels.}, $\xi_M^{1/2}$ can be upper bounded by a constant $C_k$ independent of $M$. IGP-BUCB and GP-BTS, if initialized with this scheme, attain the following regret bound.
\begin{mycorollary}[Informal regret bounds for IGP-BUCB and GP-BTS using initialization]
\label{cor:initialization}
Using the initialization scheme described above, the regret bound for IGP-BUCB is $O\Big(C_k\sqrt{ T}\big(B\sqrt{\gamma_T}+\gamma_T\big)\Big)+2BT^{init}$ with high probability and for GP-BTS is $O\Big(C_k\sqrt{Td\ln (BdT)}\left(B\sqrt{\gamma_T}\Big)+\gamma_T\right)+2BT^{init}$ with high probability. 
\end{mycorollary}
\begin{proof}
The proof is similar to \citet[Theorem 5]{desautels2014parallelizing}. The sum of the regrets
over $T$ rounds is split over the first $T^{init}$ rounds and the remaining $T-T^{init}$ rounds. Cumulative regret over the first $T^{init}$ rounds is upper bounded by $2BT^{init}$, since by our hypothesis $f(x) \le B$ for all $x \in D$. Cumulative regret for IGP-BUCB over the next $T-T^{init}$ rounds is $O\Big(C_k\big(B\sqrt{\gamma_{T-T^{init}}}+\gamma_{T-T^{init}}\big)\sqrt{T-T^{init}}\Big)=O\Big(C_k\sqrt{ T}\big(B\sqrt{\gamma_T}+\gamma_T\big)\Big)$, since $\gamma_T$ is a non decreasing function. Combining the two terms gives us the desired bound for IGP-BUCB. The bound for GP-BTS follows from a similar argument.
\end{proof}
Despite the dependence on the initialization scheme and the constant term $C_k$, Corollary \ref{cor:initialization} is encouraging as $T^{init}$ is roughly of the order of $\gamma_M$, which again is sublinear in $M$. Thus, as long as the batch size $M$
does not grow too quickly, the second terms in the regret bounds are dominated by the first terms. Hence the regret bounds
of IGP-BUCB and GP-BTS are almost as good as (upto a constant factor) those of their strictly sequential versions, i.e. IGP-UCB and GP-TS respectively \cite{pmlr-v70-chowdhury17a}.

\end{appendix}


\end{document}